%% file: ijcai24.tex
\documentclass{article}
\usepackage{ijcai24}

\usepackage{times}
\usepackage{soul}
\usepackage{url}
\usepackage[hidelinks]{hyperref}
\usepackage[utf8]{inputenc}
\usepackage[small]{caption}
\usepackage{graphicx}
\usepackage{amsmath}
\usepackage{amsthm}
\usepackage{booktabs}
\usepackage{algorithm}
\usepackage{algorithmic}
\usepackage[switch]{lineno}

\usepackage[table,xcdraw]{xcolor}
\usepackage{microtype}
\usepackage{subfigure}
\usepackage{multirow}
\usepackage{bbm}
\usepackage{dsfont}
\usepackage{bm}
\usepackage{amssymb}
\usepackage{mathtools}
\usepackage[capitalize,noabbrev]{cleveref}

\newtheorem{example}{Example}
\newtheorem{theorem}{Theorem}
\newtheorem{assumption}{Assumption}
\newtheorem{remark}{Remark}
\newtheorem{lemma}{Lemma}

\usepackage[textsize=tiny]{todonotes}

\usepackage{authblk}

\def\<{\left\langle} 
\def\>{\right\rangle}
 
\providecommand{\argmin}{\mathop{\rm argmin}}

\newcommand{\g}{\mathbf{g}}
\newcommand{\R}{\mathbb{R}}
\newcommand{\E}{\mathbb{E}}
\newcommand{\veps}{\bm{\epsilon}}
\newcommand{\vthe}{\bm{\theta}}

\title{Hard-Thresholding Meets Evolution Strategies in Reinforcement Learning}

\author[1]{Chengqian Gao\thanks{Equal contribution}}
\newcommand\CoAuthorMark{\footnotemark[\arabic{footnote}]} 
\newcommand\CoRespondingMark{\footnotemark[\arabic{footnote}]} 

\author[1]{William de Vazelhes\protect\CoAuthorMark}
\author[1]{Hualin Zhang}
\author[1, 2]{Bin Gu\thanks{Corresponding to: jsgubin@gmail.com, zhiqiang.xu@mbzuai.ac.ae}}
\author[1]{Zhiqiang Xu\protect\CoRespondingMark}

\affil[1]{Mohamed bin Zayed University of Artificial Intelligence, UAE}
\affil[2]{School of Artificial Intelligence, Jilin University, China}

\begin{document}

\maketitle

\begin{abstract}
    Evolution Strategies (ES) have emerged as a competitive alternative for model-free reinforcement learning, showcasing exemplary performance in tasks like Mujoco and Atari. Notably, they shine in scenarios with imperfect reward functions, making them invaluable for real-world applications where dense reward signals may be elusive. Yet, an inherent assumption in ES—that all input features are task-relevant—poses challenges, especially when confronted with irrelevant features common in real-world problems. This work scrutinizes this limitation, particularly focusing on the Natural Evolution Strategies (NES) variant. We propose NESHT, a novel approach that integrates Hard-Thresholding (HT) with NES to champion sparsity, ensuring only pertinent features are employed. Backed by rigorous analysis and empirical tests, NESHT demonstrates its promise in mitigating the pitfalls of irrelevant features and shines in complex decision-making problems like noisy Mujoco and Atari tasks\footnote{Code available at \protect\url{https://github.com/cangcn/NES-HT}}.
\end{abstract}

\input{chapters/1_introduction}
\input{chapters/2_preliminaries}
\input{chapters/3_nesht}
\input{chapters/4_analysis}
\input{chapters/5_experiments}
\input{chapters/6_conclusion}
\clearpage
\bibliographystyle{named}
\bibliography{ijcai24}

\input{chapters/9_appendix}
\end{document}

%% file: chapters/1_introduction.tex
\section{Introduction}

Evolution Strategies (ES) offer a compelling alternative for model-free reinforcement learning. Many studies proved the effectiveness of the ES algorithm in addressing complex decision-making problems, such as Mujoco and Atari tasks~\cite{DBLP:journals/corr/SalimansHCS17,DBLP:journals/corr/abs-1712-06567,DBLP:conf/nips/ManiaGR18}, 
and, particularly, its remarkable proficiency in dealing with problems where imperfect reward functions are demanded such as sparse reward signals or delayed feedback~\cite{DBLP:journals/corr/SalimansHCS17,DBLP:journals/corr/abs-2110-01411,DBLP:journals/corr/Francois-LavetF15,DBLP:journals/fcsc/QianY21}. This capability is particularly appealing in real-world applications where acquiring dense reward signals may be expensive or unachievable.

However, the ES works under a potentially oversimplifying assumption, i.e., every input feature is inherently relevant to the task at hand, which could lead to poor performance in its application to real-world decision-making systems. For example, for an autonomous driving system, which receives pixel inputs from onboard cameras, while its main objective is to ensure safe navigation, the data stream might unintentionally include features irrelevant to driving decisions, such as the vehicle color. 
The inclusion of such irrelevant features in the learning process can not only result in unnecessarily large model sizes, but also lead to sub-optimal decisions. 
Although the detrimental effects of task-irrelevant features have been noticed across deep learning~\cite{DBLP:journals/jmlr/HubaraCSEB17,DBLP:conf/mlsys/BlalockOFG20,DBLP:journals/csur/ChenNRWS23}, reinforcement learning~\cite{DBLP:conf/ijcai/SokarMMPS22,DBLP:conf/atal/GrootenSDMTPM23}, and unsupervised learning~\cite{DBLP:journals/tip/LiT15} research, it remains unclear how the noise features affect the performance of the ES algorithm. 

This work aims to regularize the ES algorithm with sparsity, expecting that the obtained sparse policies can automatically select and utilize a small but necessary portion of available features. Specifically, we focus on the Natural Evolution Strategies (NES) \cite{DBLP:journals/jmlr/WierstraSGSPS14}, a prevalent variant of the ES algorithm. NES estimates gradients, by only evaluating the objective function, for the objective optimization where irrelevant features  can potentially drive the training process to end up with poor policies. 
To mitigate the impact of task-irrelevant features, we introduce the Hard-Thresholding (HT) operator~\cite{blumensath2009iterative} to the NES framework, given its popularity and simplicity for performing $L_0$ constrained optimization at a desired sparsity level. 
However, the HT was originally developed for optimization problems where the gradient comes with a closed-form expression~\cite{DBLP:conf/icml/GargK09,DBLP:journals/tit/NguyenNW17}. The compatibility and effectiveness of the HT operator with the NES estimate of the gradient remains unclear. To the best of our knowledge, this is the first time to propose sparsity-induced natural evolution strategies.

We begin by examining the negative effects of task-irrelevant observations on the NES. We find that the inclusion of such irrelevant features increases the randomness of the reward function, resulting in a higher variance of the estimated gradient. Consequently, it hinders convergence to the optimal policy.
To address the issue of irrelevant features, we present the NESHT, which seamlessly integrates the HT operator into the NES algorithm. The \textit{modus operandi} of the NESHT is straightforward: the parameters are truncated to retain only a specified proportion, upon each gradient descent/ascent update. 
In addition, we provide a comprehensive analysis of the convergence and complexity of the NESHT, underpinning it with the canonical assumptions of sparse learning.
This analytical deep dive effectively resolves the lingering uncertainty regarding the compatibility of natural gradients with the hard-thresholding operator.
Further, an extensive empirical study verifies the effectiveness of the NESHT, particularly in challenging noisy Mujoco environments with sparse rewards and Atari environments with pixel inputs.

%% file: chapters/2_preliminaries.tex
\section{Preliminaries}
\paragraph{Markov decision process} 
We are concerned with the reinforcement learning problem, where our objective is to optimize a policy, denoted by $\pi_{\boldsymbol{\theta}}$, parameterized by $\boldsymbol{\theta}$. 
This policy is defined over a Markov decision process represented by $M=\langle \mathcal{S}, \mathcal{A}, \mathcal{T}, d_0, \mathcal{R}, \gamma \rangle$. For each episode, the initial state $\boldsymbol{s}_0$ is sampled from the distribution $d_0$. At each time step $t$, given an observation $\boldsymbol{s_t} \in \mathcal{S}$, the policy determines an action $\boldsymbol{a_t} \in \mathcal{A}$, which then results in an immediate reward $r(\boldsymbol{s_t}, \boldsymbol{a_t}) \in \mathcal{R}$. Subsequently, the system transitions to a new observation $\boldsymbol{s_{t+1}}$ in accordance with the dynamics $\mathcal{T}$. The resulting trajectory can be presented as $\bm{\tau}=\{(\boldsymbol{s}_t, \boldsymbol{a}_t, r, \boldsymbol{s}_{t+1})\}$. In order to balance the trade-off between immediate rewards and long-term rewards, the discount factor $\gamma$ is introduced.

%% file: chapters/3_nesht.tex
\section{Decision-Making with Irrelevant Features}
We propose NESHT, equipping the Natural Evolution Strategies with the Hard-Thresholding operator, for handling task-irrelevant observations in decision-making problem. 

\subsection{The objective function}\label{sec:objfun}
Our objective is to maximize the \textit{fitness score} achieved by the policy while minimizing the \textit{impact} induced from the task-irrelevant features. We hypothesize that employing a sparse policy can effectively manage these redundant observations.

\paragraph{Fitness score} The performance of the policy can be quantified by the \textit{fitness function}, which is defined as the expected sum of rewards over its rollout trajectories: 
\begin{equation}
    F(\boldsymbol{\theta}) := \mathbb{E}_{\tau \sim d_0, \pi_{\boldsymbol{\theta}}, \mathcal{T}} f_{\tau}(\vthe), ~\text{with}~ f_{\tau}(\vthe) := \sum_{t=0}^{|\tau|} r(\boldsymbol{s_t}, \boldsymbol{a_t})
\label{eq:fitness_function}
\end{equation}
It's important to note that, in this context, the discount factor $\gamma$ is set to 1. This is in contrast to traditional RL settings where it often assumes values such as 0.99 or 0.9. Another characteristic is that the fitness function can be discontinuous \textit{w.r.t.} the policy parameters due to the randomness in environments and the complex reward function. 

\paragraph{$L_0$-constraint optimization} We propose mitigating the impact of task-irrelevant features through a sparse policy, under the premise that sparsity can effectively filter out irrelevant information present in inputs. 
Formally, our objective is to improve a policy while also constraining its complexity, i.e., the $L_0$ constrained optimization, with $\| \cdot \|_0$ denotes the $L_0$ (pseudo-)norm (number of non-zero components of a vector):
\begin{equation}
    \max_{ \boldsymbol{\theta} } F( \boldsymbol{\theta} ) \quad \textit{s.t.} \quad  \| \boldsymbol{\theta} \|_0 < k
\label{eq:objective_function}
\end{equation}

\paragraph{Why $L_0$ constraint?} 
In our context, where only a small subset of observations is task-relevant, irrelevant features can significantly degrade performance. $L_0$-constrained optimization directly enforces a constraint on the $L_0$ norm of the learned parameter vector, ensuring the sparsity of the resulting model, alluring for feature selection tasks. Unlike $L_1$-constrained optimization, which promotes sparsity but does not guarantee exact zero values, $L_0$-constrained optimization offers precise control over sparsity by allowing certain model parameters to be set exactly to zero. This capability not only enhances model interpretability but also makes it well-suited for our setting, i.e., decision-making with irrelevant observations.

\subsection{Our proposal: NESHT} 
We introduce NESHT, a solution for decision-making problems involving both task-relevant and irrelevant features. While NES and the Hard-Thresholding operator are not novel concepts individually, their compatibility when used together may raise questions. To be self-contained, we now provide brief descriptions of each.

\paragraph{NES} 
We employ the competitive NES algorithm, to optimize the policy, with the following gradient estimator:
\begin{equation}
    \nabla_{\boldsymbol{\theta}} \mathbb{E}_{\boldsymbol{\epsilon} \sim \mathcal{N}(\boldsymbol{0}, I)} F(\boldsymbol{\theta} + \sigma \boldsymbol{\epsilon}) = \frac{1}{\sigma} \mathbb{E}_{\boldsymbol{\epsilon} \sim \mathcal{N}(0, I)}  F(\boldsymbol{\theta} + \sigma \boldsymbol{\epsilon})\boldsymbol{\epsilon} 
    \label{eq:nes_gradient_estimator_with_gaussian}
\end{equation}
In NES, the gradient is approximated through sampling and serves as an approximation, bypassing challenges with non-differentiable functions or exploding gradients. For the derivation about Equation~(\ref{eq:nes_gradient_estimator_with_gaussian}), please refer to Appendix.

\paragraph{Hard-thresholding operator}
To achieve the $L_0$-constrained optimization described as Equation~(\ref{eq:objective_function}), we introduce the hard-thresholding operator into NES. It truncates the parameter vector, retaining only $k$ components with the most significant absolute magnitudes, represented as $\operatorname{trunc}(\boldsymbol{\theta}, k)$, or, more succinctly, as $\operatorname{trunc}(\boldsymbol{\theta})$. While incorporating HT into NES is straightforward, the compatibility between HT and NES remains an open question. 

\paragraph{Compatibility concerns} 
To establish the convergence of NESHT, it is essential to demonstrate the convergence of the hard-thresholding algorithm for non-convex and discontinuous $F$, with a gradient estimated as in \eqref{eq:nes_gradient_estimator_with_gaussian} via the NES algorithm.
In the literature, \cite{xu2019non} proved the convergence of stochastic algorithms in the case of non-convex objective functions $F$, for a non-convex proximal term which can be taken as the indicator function of the set of all $k$-sparse vectors (i.e. the $L_0$ pseudo-ball). This proof of convergence applies to stochastic hard-thresholding algorithms. However, their analysis assumes Lipschitz-smoothness of $F$ and considers a general stochastic estimator of the gradient. Therefore, it does not account for the specific errors introduced by the gradient estimator from \eqref{eq:nes_gradient_estimator_with_gaussian}. More recently, the work of \cite{metel2023sparse}, analyzes the convergence of zeroth-order methods (similar to evolutionary strategies) for a Lipschitz-continuous and non-convex function $F$.  However, in our case, $F$ is discontinuous in general. Thus, to the best of our knowledge, the convergence of evolutionary strategies in such setting remains an open question. In the next section, we address this question by demonstrating that, under mild assumptions, proper convergence of Algorithm \ref{alg:nes+ht} is guaranteed.

%% file: chapters/4_analysis.tex
\begin{algorithm}[tb]
\caption{NES with Hard-Thresholding}
\label{alg:nes+ht}
    \begin{algorithmic}
    \STATE {\bfseries Input:} \\
    $\alpha$ - Learning rate, \\
    $\boldsymbol{\theta_0}$- Initial policy parameters in $\mathbb{R}^d$, \\
    $n$ \ - Population size, \\
    $N$ \ - Number of rollouts collected for each agent, \\
    $\sigma$ \ - Noise standard deviation, \\
    $k$ \  - Number of parameters to be kept.
    \FOR{$t =0, 1, 2, ... T - 1$}
        \FOR{$i=1$, ..., $n$}
        \STATE Sample a Gaussian perturbation $\boldsymbol{\epsilon}_i \sim \mathcal{N}(0, I)$ .
            \FOR{$j=1$, ..., $N$}  
            \STATE Sample a rollout $\tau^{\veps_i}_j$
            \STATE Compute returns $f_{\tau^{\veps_i}_j}(\boldsymbol{\theta}_t + \sigma \boldsymbol{\epsilon_i})$ 
            \ENDFOR
        \ENDFOR
        \STATE Set $\boldsymbol{\theta}_{t+\frac{1}{2}} \leftarrow \boldsymbol{\theta}_t + \frac{\alpha}{n N \sigma}\sum_{i=1}^n \sum_{j=1}^N f_{\tau^{\veps_i}_j}(\boldsymbol{\theta}_t + \sigma \boldsymbol{\epsilon_i}) \veps_i$
        \STATE Truncate the parameters: $\boldsymbol{\theta}_{t+1} \leftarrow \operatorname{trunc}(\boldsymbol{\theta}_{t+\frac{1}{2}}, k)$
    \ENDFOR
\end{algorithmic}
\end{algorithm}

\section{Convergence Analysis} \label{sec:analysis}
The integration of NES with HT is detailed in Algorithm~\ref{alg:nes+ht}, where the hard-thresholding operator is applied to the learned parameters after each update. In this section, we provide a proof of convergence for NES combined with Hard-Thresholding, i.e., our NESHT, addressing the compatibility concern.
Additionally, we would like to highlight that our analysis can also cover the case where no hard-thresholding operator is used (it only suffices to take the proximal term $r$ in our proof of Theorem~\ref{thm:thm} in Appendix to be the constant zero): to our knowledge, such a proof of convergence for NES for general discontinuous functions $F$ (which correspond to a realistic reinforcement learning setting) is the first in the literature, and we hope that such a result, as well as the subsequent remarks and discussions on the influence of each parameter on the convergence rate (bound on the expected reward $B$, dimension $d$, etc.) can be of interest to the NES community.

\subsection{Assumptions}
To proceed with the proof of convergence of NESHT, we will need the following assumptions below.
\begin{assumption}[Boundedness of $F$]\label{ass:bound}
    The fitness function $F$ is bounded on its domain, that is, there exists a universal constant $B > 0$ such that:
    $$\forall \vthe \in \mathbb{R}^d: |F(\vthe)| \leq B$$
\end{assumption}
\begin{remark}
$F(\vthe)$ represents the expected rewards obtained by executing policy $\pi_{\vthe}$. 
The boundedness assumption is typically reasonable since immediate rewards do not tend to infinity, and evaluation trajectories always have finite lengths. Importantly, this assumption remains valid even when dealing with task-irrelevant features.
\end{remark}

Additionally, we will need the following assumption on the variance of the cumulative reward, for a given parameter vector $\vthe$.
\begin{assumption}[Bounded variance of $f_{\tau}$]\label{ass:pol}
We posit the existence of a universal constant $C > 0$ such that the variance of the cumulative reward for any $\vthe \in \{\vthe_{0}, \vthe_{\frac{1}{2}}, ..., \vthe_{T-\frac{1}{2}},  \vthe_{T}\}$ is bounded by $C$, i.e.:
$$
\mathbb{E}_{\tau} \left[ | f_{\tau}(\vthe) - F(\vthe) |^2 \right] \leq C.
$$
\end{assumption}

\begin{remark}
    Assumption~\ref{ass:pol} reflects the inherent randomness from both the policy, whether it is deterministic or stochastic, and the environment, which introduces randomness through factors such as the dynamics $\mathcal{T}$, the reward function $r(s,a)$, and the initial distribution of states $d_0$. 
    Also, please note that if the reward and the episode length are limited, as is usually the case in RL, then Assumptions~\ref{ass:bound} and \ref{ass:pol} are satisfied. 
    An observant reader may notice that the inclusion of task-irrelevant features unavoidably leads to an increase in the constant $C$ due to the introduction of randomness. As we will see later, this increase hampers the convergence of NES algorithms.
\end{remark}

\subsection{Smoothness}
Since $F$ can be discontinuous in general, maximizing $F$ directly is impossible with evolutionary strategies. For instance if $F$ is Dirac-like, such as $F(\bm{\theta}) = \begin{cases}
    1 ~\text{if}~ \vthe = \bm{0}\\
    0 ~\text{otherwise}
\end{cases}$, the probability (for a given $\bm{\theta}$), to successfully sample an $\bm{\epsilon}$ such that $F(\bm{\theta} + \sigma \bm{\epsilon}) = 1$ is actually zero, which means the parameters will be updated with probability zero. 
However, we can instead analyze the convergence of a smoothed version of $F$, $F_{\sigma}$, defined below:  
$$F_{\sigma}(\vthe) := \mathbb{E}_{\boldsymbol{\epsilon} \sim \mathcal{N}(0, I)}
F(\boldsymbol{\theta} + \sigma \boldsymbol{\epsilon})$$
Note that $F_{\sigma}$ converges towards $F$ for small $\sigma$ in terms of \textit{eh-convergence}, as described in Theorem 3.2 from \cite{yu1992minimization}. 
The first step, to derive the convergence rate of our algorithm with $F_{\sigma}$, is to prove that $F_{\sigma}$ is smooth, and to derive its smoothness constant, which we then use in a proof framework similar to \cite{xu2019non}.

\begin{lemma}\label{lem:smoothness}
Under Assumption~\ref{ass:bound}, $F_{\sigma}$ is Lipschitz-smooth (i.e. its gradient is Lipschitz-continuous), with a smoothness constant $L = \frac{(d + 1)B}{\sigma^2}$, that is, such $L$ verifies: 
$$\forall \vthe_1, \vthe_2 \in (\R^d)^2: \|\nabla F_{\sigma}(\vthe_1) - \nabla F_{\sigma}(\vthe_2)\| \leq L \| \vthe_1 - \vthe_2\|$$
\end{lemma}

\begin{proof}
    Proof in Appendix.
\end{proof}
For discontinuous functions $F$, the fact that $F_{\sigma}$ is smooth was already known before in the literature (see e.g. \cite{ermoliev1995nonsmooth}). However, such works did not provide an explicit formula for the smoothness constant $L$. Here, for the first time in the literature (to the best of our knowledge), using the boundedness assumption on $F$, we could derive an explicit formula for the smoothness constant $L$.

One can therefore see that $L$ is proportional to both the bound of the fitness function, $B$, and the dimension of the policy parameters, $d$, while being inversely proportional to the variance $\sigma^2$. In Section 4.4, we will observe the role of such smoothness constant $L$: the smaller it is, the faster the NES algorithm will converge.

\subsection{Error of the gradient estimator} \label{sec:errgrad}
We now consider the gradient estimator with a general population of $n$ random perturbations, and a number of rollouts of $N$ for each perturbation. More precisely, assume that we sample $n$ random directions $\{\veps_i\}_{i=1}^n:= \{\veps_1, ..., \veps_n\}$ independently and identically distributed, and that for each of these random directions $\veps_i$, we sample we sample $N$ rollouts  $\{\tau_j^{\veps_i}\}_{j=1}^N:=\{\tau^{\veps_i}_1, .., \tau^{\veps_i}_N \}$ independently and identically distributed, to obtain a final collection of rollouts $\{\{\tau^{\veps_i}_j\}_{j=1}^N\}_{i=1}^n$ , and to get $N \times n$ gradient estimators $\hat{g}_{\sigma, \veps_i, {\tau}^{\veps_i}_j}, (i, j) \in [n] \times [N]$ defined below:
$$ \hat{g}_{\sigma, \veps_i, \tau_j^{\veps_i}} (\vthe) := \frac{1}{\sigma}f_{\tau_j^{\veps_i}}(\vthe + \sigma \veps_i) \veps_i$$
which we aggregate in the following estimator:
$$ \bar{g}_{\sigma, \{\veps_i\}_{i=1}^n, \{\{\tau^{\veps_i}_j\}_{j=1}^N\}_{i=1}^n}(\vthe)  := \frac{1}{nN} \sum_{i=1}^n  \sum_{j=1}^N \hat{g}_{\sigma, \veps_i, \tau^{\veps_i}_j} (\vthe) $$

\begin{lemma}\label{lem:boundvar}
Under Assumptions \ref{ass:bound} and \ref{ass:pol}, the estimator above is an unbiased estimate of the gradient of the smoothed function $F$, and its variance is bounded, more precisely, for any $\vthe \in \{\vthe_{0}, \vthe_{\frac{1}{2}}, ..., \vthe_{T-\frac{1}{2}},  \vthe_{T}\}$:   
$$\E \bar{g}_{\sigma, \{\veps_i\}_{i=1}^n, \{\{\tau^{\veps_i}_j\}_{j=1}^N\}_{i=1}^n} (\vthe) = \nabla_{\vthe} F_{\sigma}(\vthe)$$
$$ \E \|\bar{g}_{\sigma, \{\veps_i\}_{i=1}^n, \{\{\tau^{\veps_i}_j\}_{j=1}^N\}_{i=1}^n} (\vthe) - \nabla_{\vthe} F_{\sigma}(\vthe) \|^2 \leq  \frac{C d}{N\sigma^2} +  \frac{d B^2}{n\sigma^2} $$
\end{lemma}

\begin{proof}
    See Appendix. We begin by examining the unbiasedness (using a standard proof) and variance (using a novel proof up to our knowledge) of the gradient estimator for a single perturbation, i.e., $\hat{g}_{\sigma, \veps_i, \tau_j^{\veps_i}}$. We then generalize our results to account for multiple perturbations (\(n\)) and rollouts (\(N\)).
\end{proof}

\paragraph{Advantages of NESHT: reduction in constant $C$} 
We present here a formal explanation for the superiority of NESHT over NES in the lens of constant $C$. Thanks to hard-thresholding, along training, $\vthe_t$ and $\vthe_{t+{\frac{1}{2}}}$ remain in the space of $k$-sparse vectors (up to small perturbations $\sigma \veps$), whereas they could live anywhere in $\R^d$ in the case of NES. 
Based on the hypothesis that the hard-thresholding operation effectively selects relevant features (which we have verified experimentally in Section \ref{sec:hypothsis}), NESHT can successfully mitigate the impact of irrelevant features and reduces the value of $C$. To illustrate this, one can consider the following scenario.

\begin{example}\label{ex:1}
Consider a one-step decision-making experiment, with linear policy, and fitness score given as: 
$f_{\tau}(\vthe) := \bm{\bm{x}}^{\top} (\vthe - \vthe^*)$, where $\vthe^*$ is a $k$-sparse vector, with $S \subseteq [d]$ being the set of coordinates of its non-zero components, i.e., the relevant features. In addition, $\bm{\bm{x}}$ is the input state, which we assume follows a normal distribution $\mathcal{N}(\bm{0}, \sigma\bm{I}_{d \times d})$ for $\sigma >0$ ($\bm{I}_{d \times d}$ denoting the identity). 
We then have, for any bounded policy $\vthe \in [-1, 1]^d$: 
\begin{align*}
&\E_{\bm{\bm{x}}} | f_{\tau}(\vthe) - F(\vthe) |^2  = \E_{\bm{\bm{x}}} (\vthe - \vthe^*)^{\top} \bm{\bm{x}} \bm{\bm{x}}^{\top} (\vthe - \vthe^*) \\
&=  (\vthe - \vthe^*)^{\top}  \sigma^2 \bm{I}_{d \times d} (\vthe - \vthe^* ) = \sigma^2 \|\vthe - \vthe^* \|^2
\end{align*}
\end{example}

Therefore, if there are many irrelevant components present (i.e. $|[d] \setminus S|$ is large), the episode-wise variance of $f_{\tau}$ (and its bound $C$) will be higher when $\vthe$ is dense (proportionally to $\sigma^2$). As established in Lemma~\ref{lem:boundvar}, the proper convergence of NESHT depends on this variance. The application of a hard-thresholding operator explicitly filters out some of the noisy features, introducing a bias that steers the policy towards making decisions exclusively based on sparse observations. This reduces the variance and ensures better convergence to the optimal policy. 

In practical terms, given a fixed interaction budget for $n$ and $N$, the variance of the gradient estimator may be too high for vanilla NES, causing it to fail to converge to the optimal policy. However, with the reduced variance of the gradient estimator in NESHT, as described above, convergence of the parameters $\vthe$ to a stationary point of the fitness function $F$ can be successfully ensured, as stated in Theorem~\ref{thm:thm}. Section \ref{sec:xp} provides illustrations of cases where NES fails to converge to a successful policy, but NESHT can learn a successful policy in several RL tasks. This validates our hypothesis that learning a sparse policy with NESHT can properly handle irrelevant noise in the observations.

\subsection{Convergence rate}
Equipped with Lemmas \ref{lem:smoothness} and \ref{lem:boundvar}, we can now prove the convergence of Algorithm~\ref{alg:nes+ht}, following for the most part the framework of \cite{xu2019non} for stochastic gradient descent with a non-convex function and a non-convex non-smooth proximal term, but plugging into it our novel bounds for (i) the smoothness constant of $F_{\sigma}$ and (ii) the variance of the gradient estimator $\bar{g}_{\sigma, \{\veps_i\}_{i=1}^n, \{\{\tau^{\veps_i}_j\}_{j=1}^N\}_{i=1}^n}(\vthe)$ , under our specific assumption of boundedness of $F$. Because of such non-convex and non-smooth optimization problem, convergence is proven in terms of the expected distance of the Fréchet sub-differential $\hat{\partial} (- F_{\sigma}(\bm{\theta}) + \mathds{1}_{L_0(k)}(\vthe_T) )$ to zero \cite{rockafellar1976monotone}, where $\mathds{1}_{L_0(k)}$ denotes the indicator function of the $L_0$ constraint, i.e. $\mathds{1}_{L_0(k)}(\vthe) = \begin{cases}
    0 ~ \text{if} ~ \vthe ~ \text{is $k$-sparse}\\
    + \infty ~ \text{otherwise}
\end{cases}$. Note that this is the standard way to define stationary points for non-smooth regularizers (such as sparsity constraints) (see e.g. Thm. 2 in \cite{xu2019non} or Thm. 3 in \cite{deleu2021structured}).


\begin{figure*}
    \centering
    \includegraphics[width=\linewidth]{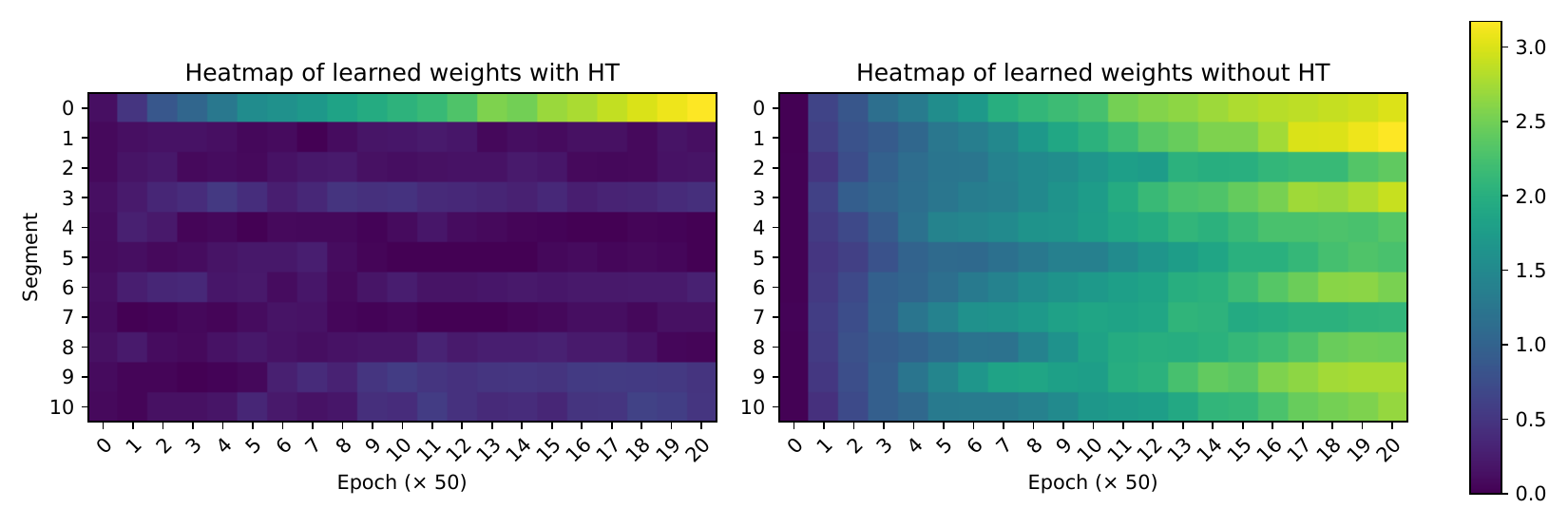}
    \caption{Heatmap illustrating the evolution of learned weights from a NESHT policy (left) and a NES policy (right) over epochs. The environment studied is Hopper-V3, with tenfold Gaussian noise. Among the 11 distinct observation segments (Y-axis), only the first (0-th) segment corresponds to the environment-provided features, while all subsequent 10 segments represent Gaussian noise (task-irrelevant features). The heatmap color indicates the norm of the learned weights. With the HT operator, only the portion of the neuron corresponding to task-relevant features (the 0-th segment) is activated. Without HT, NES struggles with task-irrelevant features, leading to poor performance.}
    \label{fig:weights_heatmap}
\end{figure*}

\begin{theorem}\label{thm:thm}
    Under Assumption~\ref{ass:bound} and \ref{ass:pol}, run Algorithm~\ref{alg:nes+ht}, with $\alpha=\frac{c}{L}\left(0<c<\frac{1}{2}\right)$, a number of iterations $T=2 c_2 B /\left(\alpha \varepsilon^2\right)$ and $N \geq \frac{4 c_1 d C}{\sigma^2 \varepsilon^2}$ and $n \geq \frac{4 c_1 d B^2}{\sigma^2 \varepsilon^2}$ for $t=0, \ldots, T-1$, then the output $\vthe_T$ of Algorithm~\ref{alg:nes+ht} satisfies
$$
\E\left[\operatorname{dist}\left(\bm{0}, \hat{\partial} \left( - F_{\sigma}\left(\vthe_T\right) + \mathds{1}_{L_0(k)}(\vthe_T)\right) \right)\right]  \leq \varepsilon,
$$
 where $c_1 = \frac{2c(1-2c)+2}{c(1-2c)} $, and $c_2=\frac{12-8c}{1-2c} $, and where $\operatorname{dist}(\bm{z}, S)$ is the distance of a set $S$ to a point $\bm{z}$, defined as the minimal Euclidean distance of any point in $S$ to $\bm{z}$. In particular in order to have $\mathbb{E}\left[\operatorname{dist}\left(0, \hat{\partial}  (- F_{\sigma}(\bm{\theta}) + \mathds{1}_{L_0(k)}(\vthe_T) ) \right)\right] \leq \varepsilon$, that is, in order to ensure convergence to a stationary point, it suffices to set $T=O\left(1 / \varepsilon^2\right)$.
\end{theorem}
\begin{proof}
    Proof in Appendix.
\end{proof}
\begin{remark}
    As per Theorem \ref{thm:thm}, we can see that a large smoothing radius $\sigma$ will ease convergence, as it allows one to evaluate fewer random perturbations and rollouts. However, the counterpart is that the function optimized $F_{\sigma}$ may be further away from the true function $F$. 
\end{remark}

\begin{remark}[Overall complexity]
From Theorem~\ref{thm:thm}, to ensure convergence to a stationary point up to tolerance $\varepsilon$, we need to take $N = O(\frac{dC}{\sigma^2 \varepsilon^2})$, $n = O(\frac{d B^2}{\sigma^2 \varepsilon^2})$ , and $T = O(\frac{B}{\alpha \varepsilon^2}) \overset{(a)}{=} O(\frac{BL}{\varepsilon^2}) \overset{(b)}{=} O(\frac{B^2 d}{\varepsilon^2 \sigma^2})$, where (a) follows from the definition of $\alpha$ from Theorem~\ref{thm:thm}, and (b) follows from  Lemma~\ref{lem:smoothness}. Therefore, the overall number of episodes needed to ensure convergence is $T N n = O(\frac{d^3 B^4 C}{\sigma^6 \varepsilon^6})$. Note however that if one has access to a massively parallel device able to run in parallel $N n$ simulations, which is very common in RL settings (e.g. as in \cite{DBLP:journals/corr/SalimansHCS17}), the time complexity of the whole optimization process is simply $T =  O(\frac{B^2 d}{\varepsilon^2 \sigma^2})$.

\end{remark}

%% file: chapters/5_experiments.tex
\section{Experiment}\label{sec:xp}
The design of NESHT is based on two central premises: (1) weights and biases corresponding to task-irrelevant features can be set to zero by the hard-thresholding operator, and (2) the hard-thresholding operator is compatible with the NES algorithm. 
In this section, we design experiments to address the following questions:
\begin{itemize}
    \item[1)] Does HT truly capture task-irrelevant observations? 
    \item[2)] Can the NESHT policy outperform other solutions? 
    \item[3)] Can the effect of HT be extended to visual tasks?
    \item[4)] How does the HT ratio affect performance?
\end{itemize}

\subsection{Experimental Setups}

We perform evaluations on two popular RL protocols, Mujoco~\cite{todorov2012mujoco} and Atari~\cite{bellemare2013arcade} environments. 

\paragraph{Mujoco setups} 
Observations in the Mujoco continuous environment are represented as floating-point values. Detecting redundancy features in these observations is challenging due to the complexity of the environment dynamics. To simulate decision-making in the presence of task-irrelevant features, we \textit{concatenate} Gaussian noise with the environment-provided observations. Additionally, we set 90\% of the immediate rewards to zero, replicating a more challenging real-world scenario characterized by an imperfect reward function. In the analysis section, we use the notation $k$ to represent the number of parameters to be retained. In the experiment section and in our implementation, we prefer $\beta$ to denote the hard-thresholding ratio, which refers to the ratio of activated neurons.

\paragraph{Protocol for Mujoco tasks} We primarily base our implementation on the framework outlined in \cite{DBLP:journals/corr/SalimansHCS17}. In this context, we use i.i.d. Gaussian perturbations in the parameter space to estimate the gradient. As a result, the natural gradient simplifies to the plain gradient, as shown in Equation~\ref{eq:nes_gradient_estimator_with_gaussian}. It is worth noting that we choose a linear policy for NESHT since it is easier to train and has been shown to be expressive enough for such tasks \cite{DBLP:conf/nips/ManiaGR18}.

\paragraph{Atari} Beside the Mujoco benchmarks, we include the more challenging Atari games with pixel inputs to answer the question of whether the hard-thresholding algorithm can handle irrelevant features in visual observations. We use the full screen of the Atari game as input (110x84 pixels). This includes not only the playing area, but also other task-irrelevant features, such as the scoreboard and backgrounds. 
Notably, we employ a CNN module for extracting latent features from the pixel inputs for Atari Games only, alone with a linear layer mapping them to the action space. A prevalent challenge associated with NES is their sample efficiency. In Atari experiments, we mirror the training configuration in \cite{DBLP:journals/corr/SalimansHCS17}. Specifically, we train the policy for a duration of 1 hour using a 500-core machine. Furthermore, we set an upper limit on the interaction budget at 10M steps.

\begin{table*}[]
    \centering
    \begin{tabular}{l l r r r r r r r}
    \toprule
        \textbf{Env Name} &  \textbf{Noise Ratio} & \textbf{TRPO} & \textbf{DDPG} & \textbf{TD3} & \textbf{Vanilla NES} & \textbf{ANF-SAC} & \textbf{NES-$L_1$} & \textbf{NES+HT(Ours)}\\
     \midrule
        \multirow{3}{4em}{\small{HalfCheetah}}
        & $\times 5$  & 198.3 & 1369.8 & 665.7 & 819.2 & 28.8 & 1152.7 &\textbf{1851.8}\\
        & $\times 10$ & 31.6 & 1285.2 & 197.7 & 805.9 & 8.5 & 918.4 &\textbf{1722.5}\\
        & $\times 20$ & 19.2 & 843.4 & 198.5 & 773.4 & 8.4 & 669.1 &\textbf{1213.9}\\
    \midrule
        \multirow{3}{4em}{Hopper} 
        & $\times 5$  & 266.9 & 917.6 & 972.5 & 803.7 & 1014.3 & 679.4 &\textbf{1354.5} \\
        & $\times 10$ & 43.7 & 824.3 & 918.0 & 241.2 & 1010.7 & 204.5 &\textbf{1187.1} \\
        & $\times 20$ & 23.0 & 809.5 & \textbf{991.2} & 147.7 & 1015.4 & 62.9 &\textbf{1006.6} \\
    \midrule
        \multirow{3}{4em}{Walker2d} 
        & $\times 5$  & 441.9 & 1030.1 & 1000.7 & 784.0 & 986.5 & 745.2 & \textbf{1043.8}\\
        & $\times 10$ & 420.3 & 907.3 & 559.7 & 384.7 & \textbf{966.6} & 364.2 &940.3 \\
        & $\times 20$ & 229.9 & 780.6 & 485.8 & 240.6 & \textbf{960.2} & 33.8 & 675.2  \\
    \bottomrule
    \end{tabular}
    \caption{Comparison on the performance of RL and ES algorithms on Mujoco locomotion tasks with varying levels of noise. The term \textbf{Noise Ratio} indicates the amount of redundancy observations (Gaussian noise). For example, $\times 5$ signifies that the dimension of Gaussian noises is five times that of the environment-provided observations. For RL algorithms, we train for 1 million environment steps while 10 million steps for ES approaches. We report the average scores received by last 10 evaluations across 20 random seeds. Notably, the results for the NESHT algorithm are from runs with a fixed hard-thresholding ratio, $\beta=0.9$.}
    \label{tab:exp_mujoco}
\end{table*}
\begin{figure*}[!ht]
    \centering
    \includegraphics[width=\linewidth]{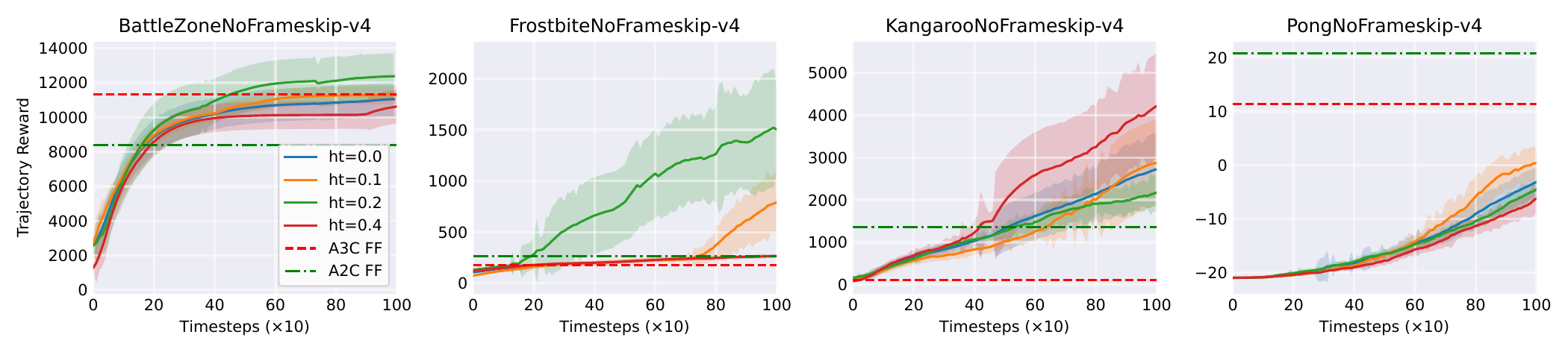}
    \hspace{-0.5cm}
    \caption{Comparison of NES in four representative Atari tasks with and without HT. We report the outcomes of NESHT under varying hard-thresholding ratios, based on 20 random seeds. Results from the A3C and A2C algorithms are adapted from~\protect\cite{DBLP:journals/corr/SalimansHCS17}.}
    \label{fig:atari_results}
\end{figure*}

\subsection{Efficacy of the hard-thresholding operator}
\label{sec:hypothsis}
We begin by evaluating the effectiveness of the hard-thresholding operator in identifying and truncating task-irrelevant features. To do this, we devise a toy example using the Mujoco Hopper-V3 environment with added noise. Specifically, the environment-provided 11-dimension features are augmented with 110-dimension i.i.d. Gaussian noise. We set the hard-thresholding ratio to 0.9, ensuring that only the top 10\% of the large learned weights are retained. 

The main focus of our analysis is to examine the norm of the learned weights across each \textit{segment} of observations. As shown in Figure~\ref{fig:weights_heatmap}, the first segment (index 0) relates to the features provided by the environment. This visualization showcases the weight norms ($L_1$ norm) for the genuine features (0-th segment) and the task-irrelevant features (remaining 10 segments). By employing the hard-thresholding operation iteratively, the learned weights for relevant features remain large, while those for irrelevant features are truncated.
By comparing results with and without the use of the hard-thresholding operator in Figure~\ref{fig:weights_heatmap}, we offer empirical evidence of its utility, i.e., sparse policy learned from the NESHT can select relevant features. Notably, it's essential to apply hard-thresholding iteratively during each update rather than just once post-learning.

\subsection{Performance on noised Mujoco tasks}
As illustrated in the preceding section, the hard-thresholding operator can successfully assist the NES algorithm in capturing task-relevant features. In this section, we evaluate the performance of the proposed NESHT algorithm in comparison to the original NES algorithm, as well as other commonly employed methods for decision-making tasks, such as RL algorithms.

\begin{figure*}[!ht]
    \centering
    \includegraphics[width=\linewidth]{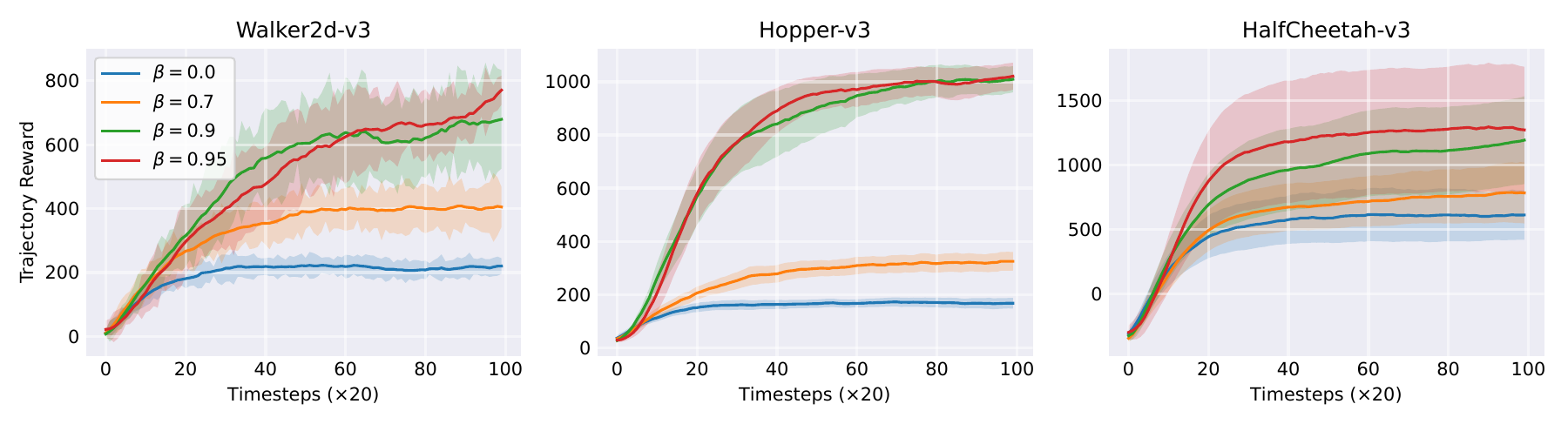}
    \hspace{-0.8cm}
    \caption{Ablation study. We assess the impact of hard-thresholding operation in the presence of Gaussian noise with a 20$\times$ noise dimension on the Mujoco environment. We vary the value of $\beta$ from 0.0 (corresponding to Vanilla NES) to 0.95 while retaining only 5\% of the neurons.}
    \label{fig:ablation}
\end{figure*}

\paragraph{Baselines}
Experiments are conducted in three popular Mujoco environments: Hopper, Walker2d, and HalfCheetah. We introduce Gaussian noise ranging from 5-fold to 20-fold, which is merged with the environment-provided features. 
Our baseline algorithms fall into three categories:
\begin{itemize}
    \item Vanilla NES policy: We follow the implementation in~\cite{DBLP:journals/corr/SalimansHCS17}, but with a modification: the agent is instantiated with a one-layer linear network.
    
    \item Classic RL algorithms: Since ES algorithms are often viewed as alternatives to RL, we include algorithms like TRPO~\cite{schulman2015trust}, DDPG~\cite{lillicrap2015continuous}, and TD3~\cite{fujimoto2018addressing} in our baselines.
    
    \item Other Solutions: One effective way to address irrelevant features is by incorporating an $L_1$-norm penalty. We explored this approach with the NES algorithm. A complexity penalty, the $L_1$-norm of the learned weights, is subtracted from the fitness score, and we refer to this as NES-$L_1$. Additionally, we include a solution in the literature of RL, ANF-SAC~\cite{grooten2023automatic}. This baseline algorithm is designed for RL with very dense rewards, while ours can handle the sparse reward signal.
\end{itemize}

It's worth noting that while NES, NESHT and NES-$L_1$ agents utilize one-layer linear networks, all other baseline algorithms for comparison are based on three-layer non-linear networks. We report the average scores across the last ten evaluations. Notably, we find ANF-SAC shows a failure after achieving impressive performance in HalfCheetah tasks.

\paragraph{Results} 
The results of our experiments are presented in Table~\ref{tab:exp_mujoco}. Our findings highlight the viability of NES as an alternative to RL approaches. Specifically, the performance of the NES agent consistently aligns with the state-of-the-art DDPG and TD3 methods. However, the NES policy shows dramatic performance degradation when confronted with increasing task-irrelevant observations. This trend highlights a harmful assumption within NES algorithms, which assumes that all features are task-relevant. Fortunately, the introduction of the hard-thresholding operator successfully mitigates this performance drop. Compared to vanilla NES, NESHT not only demonstrates enhanced resilience against irrelevant observations, but also consistently outperforms the RL baselines, emphasizing its effectiveness.

\subsection{Comparison on visual Atari benchmarks}
We expand our comparison to include Atari environments that involve visual inputs. This extension is motivated by concerns raised about the presence of artificial noises in previous evaluations. In Atari games, task-irrelevant features are from the environment and vary significantly. 
For instance, elements such as scoreboards in each game (with different locations) are considered important yet distracting observations that can hinder performance, as noted in~\cite{DBLP:journals/nature/MnihKSRVBGRFOPB15}.

We present the results on four representative Atari games in Figure~\ref{fig:atari_results}. Remarkably, we observed a significant improvement in the performance of NES algorithms when employing the hard-thresholding operator. It is important to note that the optimal hard-thresholding ratio, denoted as $\beta$, differs across these environments. We will discuss its influence in the next part. It is worth mentioning that this investigation does not include a comprehensive review of solutions in the RL or ES literature. Our main focus here is not to determine whether NES is a competitive solution (as already demonstrated by~\cite{DBLP:journals/corr/SalimansHCS17}), nor to claim that NESHT is the ultimate solution for Atari-type tasks. Instead, our goal is to explore the impact of the hard-thresholding operator on NES performance in decision-making tasks with pixel inputs.

\subsection{Hyper-parameter study} 
The NESHT algorithm, proposed in this work, introduces a critical hyperparameter known as the hard-thresholding ratio, denoted as $\beta$, which plays a fundamental role in controlling the activation of neurons. In this section, we assess its impact. We employ the noised Mujoco locomotion tasks, where we can control the amount of task-irrelevant features. Specifically, we use three environments with $\times 20$ Gaussian noises, where the noise dimension is 20 times that of the environment-provided observations.
Results depicted in Figure~\ref{fig:ablation} illustrate the effectiveness of hard-thresholding in mitigating the impact of noisy observations.The vanilla NES algorithm ($\beta=0.0$) struggles to learn from the noisy observations for policy improvement. However, as we increase $\beta$, truncating more small weights, we observe a significant performance improvement. This suggests that the hard-thresholding operation effectively filters out noise and enhances the algorithm's ability to learn from challenging, noisy environments.

%% file: chapters/6_conclusion.tex
\section{Conclusion}
Hard-thresholding emerges as a promising solution for $L_0$-constrained optimization, offering a solution to address task-irrelevant features frequently encountered in real-world scenarios. Yet, the compatibility between HT and NES gradients remains an area of active inquiry, rendering the practical implementation of such algorithms a subject of caution. In this study, we provide a theoretical foundation that establishes the convergence of the NES gradient descent (ascent) when paired with the HT operator, thus bolstering the credibility of our NESHT algorithm. Empirical assessments conducted across both Mujoco and Atari environments further substantiate the efficacy of our proposed method.

%% file: chapters/9_appendix.tex
\appendix
\onecolumn

\section{Appendix}
\subsection{NES Gradients}
\label{app:nes_gradients}
OpenAI NES~\cite{DBLP:journals/corr/SalimansHCS17} approximates the gradient with the following estimator:
\begin{equation}
    \nabla_{\boldsymbol{\theta}} \mathbb{E}_{\boldsymbol{\epsilon} \sim \mathcal{N}(0, I)} F(\boldsymbol{\theta} + \sigma \boldsymbol{\epsilon}) = \frac{1}{\sigma} \mathbb{E}_{\boldsymbol{\epsilon} \sim \mathcal{N}(0, I)} \{F(\boldsymbol{\theta} + \sigma \boldsymbol{\epsilon}) \boldsymbol{\epsilon}\}, 
\end{equation}
with $\boldsymbol{\theta}$ for the learned policy's parameters, $\boldsymbol{\epsilon}$ for the Gaussian noise on the parameter vector, and $\sigma$ to control the standard deviation. 
\begin{align*}
    \nabla_{\boldsymbol{\theta}} \mathbb{E}_{\boldsymbol{\epsilon} \sim \mathcal{N}(0, I)} F(\boldsymbol{\theta} + \sigma \boldsymbol{\epsilon}) 
    &= \nabla_{\boldsymbol{\theta}} \mathbb{E}_{\boldsymbol{x} \sim \mathcal{N}(\boldsymbol{\theta}, \sigma^2 I)} F(\boldsymbol{x})\\ 
    &= \nabla_{\boldsymbol{\theta}} \int_{\boldsymbol{x}} P(\boldsymbol{x} | \boldsymbol{\theta} , \sigma^2 I) F(\boldsymbol{x}) d \boldsymbol{x} \\
    &= \int_{\boldsymbol{x}} \nabla_{\boldsymbol{\theta}}  P(\boldsymbol{x} | \boldsymbol{\theta} , \sigma^2 I) F(\boldsymbol{x}) d \boldsymbol{x} \qquad (\text{Leibniz integral rule})\\
    &= \int_{\boldsymbol{x}} P(\boldsymbol{x}|\boldsymbol{\theta}, \sigma^2 I) \nabla_{\boldsymbol{\theta}} \log\Big[P(\boldsymbol{x} | \boldsymbol{\theta} , \sigma^2 I)\Big] F(\boldsymbol{x}) d \boldsymbol{x} \qquad \big(\text{log derivative trick}\big)\\
    &= \mathbb{E}_{\boldsymbol{x} \sim \mathcal{N}(\boldsymbol{\theta}, \sigma^2 I)}\big\{ F(\boldsymbol{x}) \nabla_\theta\big( - \dfrac{\|\boldsymbol{x} - \boldsymbol{\theta}\|_2^2}{2\sigma^2}\big) \big\} \qquad (\text{Gaussian P.D.F.})\\
    &= \frac{1}{\sigma} \mathbb{E}_{\boldsymbol{\epsilon} \sim \mathcal{N}(0, I)} \big\{ F(\boldsymbol{\theta} + \sigma\boldsymbol{\epsilon}) \cdot \boldsymbol{\epsilon} \big\}  \quad \big(\boldsymbol{x} \leftarrow \boldsymbol{\theta} + \sigma \boldsymbol{\epsilon} \big) 
\end{align*}

\subsection{Proof of Lemma~\ref{lem:boundvar}}\label{proof:vargrad}

Our proof of Lemma~\ref{lem:boundvar} is novel up to our knowledge, although it uses standard tools in optimization. We derive our Lemma~\ref{lem:var_vanilla} using a simple application of our assumptions, and derive Lemma~\ref{lem:boundvar1} with special care on the use of conditional expectations, since the rollouts $\bm{\tau}$ are sampled for a given $\veps$.

\subsubsection{Bias and Variance of the single perturbation, full expected policy gradient estimator}

Before deriving the error of the full gradient estimator (i.e. averaged over both the $n$ perturbations, and the $N$ rollouts), we first provide the bias and variance of the gradient estimator for a single random perturbation $\veps$, defined below as: 

$$ \hat{g}_{\sigma, \veps} (\vthe)= \frac{1}{\sigma}F(\vthe + \sigma \veps) \veps.$$

\paragraph{Bias:}

We first start by deriving the bias of such estimator.

\begin{lemma}\label{lem:unbiased_vanilla}
   $$ \E_{\tau, \veps} [\hat{g}_{\sigma, \veps} (\vthe)] = \nabla_{\vthe} F_{\sigma}(\vthe)$$
\end{lemma}

\begin{proof}

We proceed as in \cite{ermoliev1995nonsmooth}.

Let us denote the following $d$-dimensional isotropic Normal distribution $\phi(\veps) = \frac{1}{\left( 2 \pi \right)^{d/2}} e^{-\frac{\|\veps\|^2}{2}}$.

Note that we have:
\begin{equation}\label{eq:grad}
    \nabla \phi(\veps) = - \veps  \frac{1}{\left( 2 \pi \right)^{d/2}} e^{-\frac{\|\veps\|^2}{2}} =   - \veps \phi(\veps).
\end{equation}

Therefore:

\begin{align}
     \nabla F_{\sigma}(\vthe) &=  \nabla_{\vthe} \E F(\vthe + \sigma \veps)\nonumber\\
     &= \nabla_{\vthe} \int_{\R^d}  F(\vthe + \sigma \veps) \phi(\veps) d \veps\nonumber\\
     &\overset{(a)}{=} \frac{1}{\sigma^d}\nabla_{\vthe} \int_{\R^d}  F(\veps') \phi\left(\frac{\veps' - \vthe}{\sigma}\right) d \veps'\nonumber\\
         &=\frac{1}{\sigma^d}\nabla_{\vthe} \int_{\R^d}  F(\veps) \phi\left(\frac{\veps - \vthe}{\sigma}\right) d \veps\nonumber\\
         &\overset{(b)}{=} \frac{1}{\sigma^d} \int_{\R^d}  F(\veps) \nabla_{\vthe} \left[\phi\left(\frac{\veps - \vthe}{\sigma}\right)\right] d \veps\label{eq:refgrad}\\
         &= \frac{1}{\sigma^d} \int_{\R^d}  F(\veps) \left( - \frac{1}{\sigma} \right) \left(- \frac{\veps - \vthe}{\sigma}\right)\phi\left(\frac{\veps - \vthe}{\sigma}\right) d \veps\nonumber\\
         &\overset{(c)}{=} \int_{\R^d}  F(\sigma \veps' + \vthe) \left(  \frac{1}{\sigma} \right)\veps'\phi\left(\veps'\right) d \veps'\nonumber\\
         &=  \E_{\veps} \frac{1}{\sigma} F(\vthe + \sigma \veps) \veps = \E \hat{g}_{\sigma, \veps} (\vthe)\nonumber
\end{align}

Where in (a), we do the change of variable $\veps' = \vthe + \sigma \veps$
, in (b) we exchange integral and differentiation (as per Leibniz integral rule), which is possible in our case since $F$ is bounded per Assumption~\ref{ass:bound} (see \cite{ermoliev1995nonsmooth} (24) for instance), and in (c) we use the reverse change of variable as before ($\veps' = \frac{\veps - \vthe}{\sigma}$, so $\veps = \sigma \veps' + \vthe$
).

\end{proof}

\paragraph{Variance:} We now proceed with deriving the variance of such estimator: such result, expressed in terms of the bound $B$ on $F$, is, up to our knowledge, novel.

\begin{lemma}\label{lem:var_vanilla}
Under Assumption~\ref{ass:bound}, we have:
     $$ \E_{\tau, \veps}  \|\hat{g}_{\sigma, \veps} (\vthe)  - \nabla_{\vthe} F_{\sigma}(\vthe)\|^2 \leq \frac{d B^2}{\sigma^2} $$
\end{lemma}

\begin{proof}

With $\veps \sim N(\bm{0}, \bm{I}_{d \times d})$, we have: 

\begin{equation}
    \E \|\veps\|^2 = \E \veps^{\top} \veps = \E \sum_{i=1}^d u_i^2 = \sum_{i=1}^d \E u_i^2 = d.
\end{equation}

Using the definition of $\hat{g}_{\sigma, \veps}$ and Assumption~\ref{ass:bound}, we have that: 

$$\|\hat{g}_{\sigma, \veps}\|^2 \leq \frac{B^2}{\sigma^2} \| \veps\|^2$$
Therefore: 
$$\E \|\hat{g}_{\sigma, \veps}\|^2 \leq \frac{B^2}{\sigma^2} \E \| \veps\|^2 = \frac{B^2}{\sigma^2} d. $$

We now use the bias-variance decomposition (in norm) (for a random variable $X$, $\E \|X - \E X \|^2 = \E \| X\|^2 - \|\E X \|^2$):

For any $\vthe \in \R^d$:
$$\E \|\hat{g}_{\sigma, \veps}  (\vthe)- \nabla F_{\sigma} (\vthe) \|^2 = \E \| \hat{g}_{\sigma, \veps} (\vthe)\|^2 - \|\nabla F_{\sigma} (\vthe) \|^2  \leq \E \| \hat{g}_{\sigma, \veps} (\vthe)\|^2 = \frac{d B^2}{\sigma^2}$$

\end{proof}

\subsubsection{Bias and Variance of the single perturbation, single rollout policy gradient estimator}

We now proceed with proving the bias and variance of the policy gradient estimator for a single perturbation $\veps$, and a single rollout $\tau$, defined below as:

$$ \hat{g}_{\sigma, \veps, \tau} (\vthe) := \frac{1}{\sigma}f_{\tau}(\vthe + \sigma \veps) \veps$$

where the rollout $\tau$ is sampled for a given $\veps$ (i.e. one first samples some $\veps$ to obtain a policy parameterized by $\vthe + \sigma \veps$, and then, one samples a rollout $\tau$ from that policy).

\begin{lemma}[Bias]\label{lem:unbiased}\label{lem:bias1} The gradient estimator is unbiased:
$$ \E_{\veps, \tau}    \hat{g}_{\sigma, \veps, \tau} (\vthe)  = \nabla_{\vthe} F_{\sigma}(\vthe)$$
\end{lemma}

\begin{proof}
    By the law of total probabilities, and using Assumption~\ref{ass:pol} we have: 
    \begin{align*}
    \E_{\veps, \tau}    \hat{g}_{\sigma, \veps, \tau} (\vthe) &= \E_{\veps} \E_{\tau | \veps}   \hat{g}_{\sigma, \veps, \tau} (\vthe) =  \E_{\veps} \E_{\tau | \veps} \frac{1}{\sigma}f_{\tau}(\vthe + \sigma \veps) \veps \\
    &= \E_{\veps}  \frac{1}{\sigma} \left( \E_{\tau | \veps} f_{\tau}(\vthe + \sigma \veps)  \right) \veps \\
    &= \E_{\veps}  \frac{1}{\sigma} F(\vthe + \sigma \veps) \veps = \nabla_{\vthe} F_{\sigma}(\vthe)
    \end{align*}
    Where the last equality follows from Lemma~\ref{lem:unbiased_vanilla}.
\end{proof}

\begin{lemma}[Variance]\label{lem:boundvar1}
Assume that Assumption~\ref{ass:bound} is verified, as well as Assumption~\ref{ass:pol}.
We have, for any $\vthe \in \{\vthe_{0}, \vthe_{\frac{1}{2}}, ..., \vthe_{T-\frac{1}{2}},  \vthe_{T}\}$: 

$$ \E_{\tau, \veps} \|\hat{g}_{\sigma, \veps, \tau} (\vthe) - \nabla F_{\sigma}(\vthe) \|^2 \leq  \frac{C d}{\sigma^2} +  \frac{d B^2}{\sigma^2}$$
    
\end{lemma}

\begin{proof}
For simplicity, let us fix $\vthe \in \{\vthe_{0}, \vthe_{\frac{1}{2}}, ..., \vthe_{T-\frac{1}{2}},  \vthe_{T}\}$ and denote $ \hat{g}_{\sigma, \veps, \tau} := \hat{g}_{\sigma, \veps, \tau}(\vthe)$.

    \begin{align*}
        &\E_{\tau, \veps} \|\hat{g}_{\sigma, \veps, \tau} - \nabla F_{\sigma}(\vthe) \|^2\\
        =& \E_{\tau, \veps} \| \hat{g}_{\sigma, \veps, \tau} - \frac{1}{\sigma} F(\vthe + \sigma \veps) \veps \|^2 + \E_{\tau, \veps} \|  \frac{1}{\sigma}F(\vthe + \sigma \veps) \veps  - \nabla F_{\sigma}(\vthe)\|^2 \\
&+ 2 \E_{\tau, \veps} \langle \hat{g}_{\sigma, \veps, \tau} - \frac{1}{\sigma} F(\vthe + \sigma \veps)  \veps,   \frac{1}{\sigma}F(\vthe + \sigma \veps) \veps - \nabla F_{\sigma}(\vthe) \rangle\\
    =& \E_{\tau, \veps} \| \frac{1}{\sigma}  \left (f_{\tau}(\vthe + \sigma \veps) \veps - F(\vthe + \sigma \veps) \veps \right ) \|^2 + \E_{\tau, \veps} \|  \frac{1}{\sigma}F(\vthe + \sigma \veps) \veps  - \nabla F_{\sigma}(\vthe)\|^2 \\
&+ 2 \E_{\tau, \veps} \langle \frac{1}{\sigma}f_{\tau}(\vthe + \sigma \veps) \veps - \frac{1}{\sigma} F(\vthe + \sigma \veps)  \veps,   \frac{1}{\sigma}F(\vthe + \sigma \veps) \veps - \nabla F_{\sigma}(\vthe) \rangle\\
    =& \frac{1}{\sigma^2} \E_{\tau, \veps} |  f_{\tau}(\vthe + \sigma \veps) -  F(\vthe + \sigma \veps)| \|  \veps \|^2 + \E_{\tau, \veps} \|  \frac{1}{\sigma}F(\vthe + \sigma \veps) \veps  - \nabla F_{\sigma}(\vthe)\|^2 \\
&+ 2 \E_{\veps}  \langle \E_{\tau | \veps} \frac{1}{\sigma}f_{\tau}(\vthe + \sigma \veps) \veps - \frac{1}{\sigma} F(\vthe + \sigma \veps)  \veps,   \frac{1}{\sigma}F(\vthe + \sigma \veps) \veps - \nabla F_{\sigma}(\vthe) \rangle\\
    \overset{(a)}{=}& \frac{1}{\sigma^2} \E_{\tau, \veps} |  f_{\tau}(\vthe + \sigma \veps) -  F(\vthe + \sigma \veps)| \|  \veps \|^2 + \E_{\tau, \veps} \|  \frac{1}{\sigma}F(\vthe + \sigma \veps) \veps  - \nabla F_{\sigma}(\vthe)\|^2 \\
        \overset{(b)}{\leq}& \frac{1}{\sigma^2} C \E_{\tau, \veps}  \|  \veps \|^2 + \E_{\tau, \veps} \|  \frac{1}{\sigma}F(\vthe + \sigma \veps) \veps  - \nabla F_{\sigma}(\vthe)\|^2 \\
    \overset{(c)}{\leq} &\frac{C d}{\sigma^2} +  \frac{d B^2}{\sigma^2}
    \end{align*}

    Where (a) follows from Lemma~\ref{lem:unbiased} (which implies $\E_{\tau | \veps} \frac{1}{\sigma}f_{\tau}(\vthe + \sigma \veps) \veps - \frac{1}{\sigma} F(\vthe + \sigma \veps)  \veps = 0$), (b) follows from Assumption~\ref{ass:pol}, and (c) follows from Lemma~\ref{lem:var_vanilla}.

\end{proof}


\subsubsection{Proof of Lemma~\ref{lem:boundvar}: Bias and Variance of the averaged gradient estimator}

We can now finally proceed with proving the bias and variance of the full gradient estimator, which is the averaging of the above single random perturbation and rollout gradient estimator, over several random perturbations $\veps$ and rollouts $\tau$. We recall Lemma~\ref{lem:boundvar} in its full form, including the necessary notations, in Lemma~\ref{lem:var_total} below:

\begin{lemma}[i.e. Lemma~\ref{lem:boundvar} from subsection~\ref{sec:errgrad}]\label{lem:var_total}
Assume that we sample $n$ random directions $\{\veps_i\}_{i=1}^n:= \{\veps_1, ..., \veps_n\}$ independently and identically distributed, and that for each of those random directions $\veps_i$, we sample we sample $N$ rollouts  $\{\tau_j^{\veps_i}\}_{j=1}^N:=\{\tau^{\veps_i}_1, .., \tau^{\veps_i}_N \}$ independently and identically distributed, to obtain a final collection of rollouts $\{\{\tau^{\veps_i}_j\}_{j=1}^N\}_{i=1}^n$ , and to get $N \times n$ gradient estimators $\hat{g}_{\sigma, \veps_i, {\tau}^{\veps_i}_j}, (i, j) \in [n] \times [N]$, and to obtain the following estimator, for any $\vthe \in \{\vthe_{0}, \vthe_{\frac{1}{2}}, ..., \vthe_{T-\frac{1}{2}},  \vthe_{T}\}$: 
$$ \bar{g}_{\sigma, \{\veps_i\}_{i=1}^n, \{\{\tau^{\veps_i}_j\}_{j=1}^N\}_{i=1}^n} (\vthe)  = \frac{1}{nN} \sum_{i=1}^n  \sum_{j=1}^N \hat{g}_{\sigma, \veps_i, \tau^{\veps_i}_j} (\vthe) $$

Then , we have:  

$$\bar{g}_{\sigma, \{\veps_i\}_{i=1}^n, \{\{\tau^{\veps_i}_j\}_{j=1}^N\}_{i=1}^n}(\vthe) = \nabla_{\vthe} F_{\sigma}(\vthe)$$

and:

$$ \E \|\bar{g}_{\sigma, \{\veps_i\}_{i=1}^n, \{\{\tau^{\veps_i}_j\}_{j=1}^N\}_{i=1}^n} (\vthe) - \nabla_{\vthe} F_{\sigma}(\vthe) \|^2 \leq  \frac{C d}{N\sigma^2} +  \frac{d B^2}{n\sigma^2} $$
    
\end{lemma}

\begin{proof}
    The unbiasedness follows from the linearity of expectation and the proof of ~Lemmas~\ref{lem:bias1}, and the variance follows from the fact that using $m$ i.i.d. samples of a random variable $X$ (for some integer $m$) divides the variance of the sample mean of $X$ by $m$, in the previous proof of \ref{lem:boundvar1}.
\end{proof}

\subsection{Proof of Lemma \ref{lem:smoothness}}\label{proof:smoothness}

Such proof is, up to our knowledge, novel, and uses the bound $B$ on $F$ to derive a bound on the Hessian $\nabla^2 F_{\sigma}(\vthe)$, using properties of the spectral norm for rank-1 matrices.

\begin{proof}
We have, with $\phi(\veps) = \frac{1}{\left( 2 \pi \right)^{d/2}} e^{-\frac{\|\veps\|^2}{2}}$ (cf. Equation \ref{eq:refgrad}): 

\begin{equation}
    \label{eq:befex}
    \nabla_{\vthe} F_{\sigma}(\vthe) = \frac{1}{\sigma^{d}} \int_{\R^d} F(\veps) \nabla_{\vthe} \left[ \phi \left( \frac{\veps - \vthe}{\sigma}\right) \right]d \veps
\end{equation}

And we also have, from Equation (\ref{eq:grad}), and with $\frac{\partial}{\partial \veps} $ denoting the partial derivative with respect to $\veps$, and denoting for simplicity $\phi''$ the Hessian of $\phi$:

$$  \phi''(\veps) = \frac{\partial}{\partial \veps} (    - \veps \phi(\veps)) $$ 

Therefore:

$$ \phi''(\veps) = -\bm{I} \phi(\veps) - \veps ( - \veps^{\top} \phi(\veps)) = ( \veps \veps^{\top} -\bm{I} ) \phi(\veps) $$

In Equation (\ref{eq:befex}) above, we can exchange differentiation and integral since the gradient of the Gaussian function, which we denote by $\phi'(\vthe)$, is continuously differentiable and tends to zero faster than any polynomial of $\vthe$, and  $F(\vthe)$ is bounded according to our assumptions (and therefore grows to infinity not faster than a bounded polynomial of $\vthe$, cf. \cite{ermoliev1995nonsmooth} p. (20)). Therefore, we obtain:
:

\begin{align*}
    \nabla_{\vthe}^2 F_{\sigma} (\vthe) &=  \frac{1}{\sigma^{d}} \int_{\R^d} F(\veps) \nabla_{\bm{\theta}}^2 \left[ \phi \left( \frac{\veps - \vthe}{\sigma}\right) \right] d \veps\\
    &=  \frac{1}{\sigma^{d+2}} \int_{\R^d} F(\veps) \phi'' \left( \frac{\veps - \vthe}{\sigma}\right) d \veps\\
    &=  \frac{1}{\sigma^{d+2}}  \int_{\R^d} F(\veps) \left(  \left( \frac{\veps - \vthe}{\sigma}\right) \left( \frac{\veps - \vthe}{\sigma}\right)^{\top} - \bm{I}\right)\phi \left( \frac{\veps - \vthe}{\sigma}\right) d \veps\\
    & = \frac{1}{\sigma^{2}}  \int_{\R^d} F(\vthe + \sigma \veps) \left( \veps  \veps^{\top} - \bm{I}\right) \phi \left( \veps\right) d \veps.
\end{align*}

Therefore, we have, with $\| \cdot\|_{\text{s}}$ denoting the spectral norm:

\begin{align*}
\| \nabla^2 F_{\sigma} (\vthe)  \|_{s} &= \| \frac{1}{\sigma^{2}}  \int_{\R^d} F(\vthe + \sigma \veps) \left(  \veps  \veps^{\top} - \bm{I}\right) \phi \left( \veps\right) d \veps \|_{s}\\
&\overset{(a)}{\leq}  \frac{1}{\sigma^2}   \int_{\R^d} \| F(\vthe + \sigma \veps) \left(   \veps  \veps^{\top} - \bm{I}\right)  \|_{s}  \phi \left( \veps\right) d \veps\\
&=  \frac{1}{\sigma^2}   \int_{\R^d} |  F(\vthe + \sigma \veps)| \| \left(   \veps  \veps^{\top} - \bm{I}\right)  \|_{s}  \phi \left( \veps\right) d \veps\\
 &\leq \frac{1}{\sigma^2}   \int_{\R^d}B \| \left(  \veps  \veps^{\top} - \bm{I}\right)   \|_{s} \phi \left( \veps\right) d \veps\\
 &= \frac{B}{\sigma^2}   \int_{\R^d} \|  \veps  \veps^{\top} - \bm{I} \|_{s} \phi \left( \veps\right)  d \veps\\
 &\overset{(b)}{\leq} \frac{B}{\sigma^2}   \int_{\R^d} \left(\|  \veps  \veps^{\top}\|_{s} + \| \bm{I}\|_{s} \right) \phi \left( \veps\right)  d \veps\\
  &\overset{(c)}{=} \frac{B}{\sigma^2}   \int_{\R^d} \left(  \|\veps \|_2^2 + 1 \right) \phi \left( \veps\right)  d \veps\\
  &= \frac{B}{\sigma^2} \left[ (\E \| \veps\|_2^2) + \left( \int_{\R^d} \phi(\veps) d \veps \right) \right]\\
    &= \frac{B}{\sigma^2} \left[ d +1 \right]\\
    &= \frac{(d + 1)B}{\sigma^2}  .
\end{align*}

Where (a) follows from Jensen inequality for expectation (and  since any norm, including the spectral norm, is convex)
, and where (b) follows from the triangular inequality.
And where (c) follows from the fact that the spectral norm of $\veps \veps^{\top}$ is $\|\veps \|_2^2$ since the Singular Value Decomposition of $\veps \veps^{\top}$ is $\veps \veps^{\top} = \frac{\veps}{\| \veps\|_2} \|\veps \|_2^2 \frac{\veps^{\top}}{\| \veps\|_2}$ (therefore, the largest singular value of $\veps \veps^{\top}$, which is the spectral norm by definition, is equal to $\| \veps\|_2^2$). We can now use Lemma 1.2.2 in \cite{nesterov2018lectures} to relate such bound on the Hessian to the smoothness constant of $F_{\sigma}$.


\end{proof}

\subsection{Proof of Theorem~\ref{thm:thm}: Final convergence rate}\label{proof:cvrate}
We can now use the above results into the general framework from \cite{xu2019non},  with some additional modifications to adapt their proof to our case of rewards \textit{maximization} (and not function minimization), and to our specific proximal term, which is the indicator function of the $L_0$ pseudo-ball of radius $k$ (for which the Euclidean projection onto it is the hard-thresholding operator), as well as a few modifications where we use our boundedness assumption on $F$ (Assumption~\ref{ass:bound}) instead of Assumption 1(ii) in \cite{xu2019non}.




\begin{proof}
Let $ F_{\sigma}^{-} (\vthe) :=-F_{\sigma} (\vthe)$, then we have 
\begin{align*}
    \max_{\|\vthe\|_0 \le k} F_{\sigma} (\vthe) = \min_{\|\vthe\|_0 \le k} F_{\sigma}^{-} (\vthe)
\end{align*}

Note that the nonconvex optimization problem $\min_{\|\vthe\|_0 \le k} F_{\sigma}^{-} (\vthe)$ can be reformulated as an alternative nonconvex optimization problem, wherein a nonsmooth, nonconvex indicator function serves as a regularization term:
\begin{align*}
 \min_{\|\vthe\|_0 \le k} F_{\sigma}^{-} (\vthe) = \min_{\vthe \in \R^d} F_{\sigma}^{-} (\vthe) + r(\vthe) \quad \text{where } \quad r(\vthe):= \begin{cases}
        0, &\text{if } \|\vthe\|_0 \le k \\
        +\infty, & \text{otherwise}
    \end{cases}
\end{align*}

Note that $r(\vthe)$ is a nonconvex lower-semicontinuous function (cf. the Introduction of \cite{xu2019non} for instance).

For simplification, denote by $\g_t$ the averaged gradient estimator of $\nabla_{\vthe} F_{\sigma}^{-} (\vthe_t)$ at time step $t$, i.e.,
\begin{align*}
    \g_t = -  \bar{g}_{\sigma, \{\veps_i\}_{i=1}^n, \{\{\tau^{\veps_i}_j\}_{j=1}^N\}_{i=1}^n}( \vthe_t).
\end{align*}

Then the update rule of $\vthe_{t+1}$ is equivalent to
\begin{align*}
    \vthe_{t+1} = & \text{trunc}\left( \vthe_t - \alpha \g_t  \right) \\
    \in & \argmin_{\vthe \in \R^d} \left\{ r(\vthe) + \frac{1}{2 \alpha} \| \vthe - (\vthe_t - \alpha \g_t) \|^2 \right\} \\
    = & \argmin_{\vthe \in \R^d} \left\{ r(\vthe)  + \< \g_t, \vthe - \vthe_t \> + \frac{1}{2\alpha } \|\vthe - \vthe_t\|^2 \right\}
\end{align*}
Then we have, with $\hat{\partial}$ denoting the Fr\'echet derivative (see \cite{xu2019non,deleu2021structured} for more details, in particular the proof of Theorem 2 in \cite{xu2019non}.
):
\begin{align}
    & - (\g_t + \frac{1}{\alpha} (\vthe_{t+1} - \vthe_t)) \in \hat{\partial} r(\vthe_{t+1}), \quad \nabla F_{\sigma}^{-} (\vthe_{t+1}) - (\g_t + \frac{1}{\alpha} (\vthe_{t+1} - \vthe_t)) \in \hat{\partial} (F_{\sigma}^{-} + r)(\vthe_{t+1}), \label{eq:frech} \\
    & r(\vthe_{t+1}) + \< \g_t, \vthe_{t+1} - \vthe_t \> + \frac{1}{2\alpha } \|\vthe_{t+1} - \vthe_t\|^2 \le r(\vthe_{t}) + \< \g_t, \vthe_{t} - \vthe_t \> + \frac{1}{2\alpha } \|\vthe_{t} - \vthe_t\|^2 = r(\vthe_t) \nonumber
\end{align}
According to Lemma~\ref{lem:smoothness}, the spectrum norm of Hessian matrix of $F_{\sigma}$ is bounded by $L := \frac{(d + 1)B}{\sigma^2}$, which also implies that $F_{\sigma}^{-}$ is $L$ smooth. Then we have,
\begin{align*}
    F_{\sigma}^{-}(\vthe_{t+1}) \le F_{\sigma}^{-}(\vthe_{t}) + \<\nabla F_{\sigma}^{-}(\vthe_t), \vthe_{t+1} - \vthe_t\> + \frac{L}{2} \|\vthe_{t+1} - \vthe_t\|^2.
\end{align*}
Then we have
\begin{align}\label{eq1}
    \<\g_t - \nabla F_{\sigma}^{-} (\vthe_t), \vthe_{t+1} - \vthe_t\> + \frac{1}{2} (\frac{1}{\alpha} - L) \|\vthe_{t+1} - \vthe_t\|^2 \le (F_{\sigma}^{-} + r) (\vthe_t) - (F_{\sigma}^{-} + r) (\vthe_{t+1}).
\end{align}
By Young's inequality, 
\begin{align*}
    \frac{1}{2} (\frac{1}{\alpha} - L) \|\vthe_{t+1} - \vthe_t\|^2 \le (F_{\sigma}^{-} + r) (\vthe_t) - (F_{\sigma}^{-} + r) (\vthe_{t+1}) + \frac{1}{2L} \|\g_t - \nabla F_{\sigma}^{-}(\vthe_t)\|^2 + \frac{L}{2} \|\vthe_{t+1} - \vthe_t\|^2.
\end{align*}
Summing up the above inequality over time steps $t=0, \dots, T-1$, we have 
\begin{align}\label{eq2}
    \left( \frac{1}{2\alpha} - L \right) \sum_{t=0}^{T-1} \|\vthe_{t+1} - \vthe_t\|^2 \le & (F_{\sigma}^{-} + r) (\vthe_0) - (F_{\sigma}^{-} + r) (\vthe_{T-1}) + \frac{1}{2L} \sum_{t=0}^{T-1} \|\g_t - \nabla F_{\sigma}^{-}(\vthe_t)\|^2 \notag \\
    \le & (F_{\sigma}^{-} + r) (\vthe_0) - (F_{\sigma}^{-} + r) (\vthe_{T-1}) + \frac{1}{2L} \sum_{t=0}^{T-1} \|\g_t - \nabla F_{\sigma}^{-}(\vthe_t)\|^2 \notag \\
    = & F_{\sigma}^{-} (\vthe_0) - F_{\sigma}^{-} (\vthe_{T-1}) + \frac{1}{2L} \sum_{t=0}^{T-1} \|\g_t - \nabla F_{\sigma}^{-}(\vthe_t)\|^2 \notag\\
    \le & 2B + \frac{1}{2L} \sum_{t=0}^{T-1} \|\g_t - \nabla F_{\sigma}^{-}(\vthe_t)\|^2,
\end{align}
where the equality is due to the definition of the indicator function $r(\vthe)$ and since $\vthe_0$ and $\vthe_{T-1}$ are $k$-sparse, and the last inequality is due to Assumption \ref{ass:bound}. According to Eq.\eqref{eq1}, we also have:
\begin{align*}
    & \frac{2}{\alpha} \<\g_t - \nabla F_{\sigma}^{-} (\vthe_{t+1}), \vthe_{t+1} - \vthe_t\> + \frac{1-\alpha L}{\alpha^2} \|\vthe_{t+1} - \vthe_t\|^2 \\
    \le & \frac{2 \left( (F_{\sigma}^{-} + r) (\vthe_t) - (F_{\sigma}^{-} + r) (\vthe_{t+1}) \right)}{\alpha} - \frac{2}{\alpha} \<\nabla F_{\sigma}^{-}(\vthe_{t+1}) - \nabla F_{\sigma}^{-}, \vthe_{t+1} - \vthe_t\>.
\end{align*}
Since $2 \<\g_t - \nabla F_{\sigma}^{-} (\vthe_{t+1}), \frac{1}{\alpha} (\vthe_{t+1} - \vthe_t)\> = \| \g_t - \nabla F_{\sigma}^{-} (\vthe_{t+1}) + \frac{1}{\alpha} (\vthe_{t+1} - \vthe_t) \|^2 - \|\g_t - \nabla_{\sigma}^{-} (\vthe_{t+1})\|^2 - \frac{1}{\alpha^2}\| \vthe_{t+1} - \vthe_t \|^2 $, we have :
\begin{align*}
    & \| \g_t - \nabla F_{\sigma}^{-} (\vthe_{t+1}) + \frac{1}{\alpha} (\vthe_{t+1} - \vthe_t) \|^2 \\
    \le & \|\g_t - \nabla_{\sigma}^{-} (\vthe_{t+1})\|^2 + \frac{1}{\alpha^2}\| \vthe_{t+1} - \vthe_t \|^2 - \frac{1-\alpha L}{\alpha^2} \|\vthe_{t+1} - \vthe_t\|^2 \\
    & + \frac{2 \left( (F_{\sigma}^{-} + r) (\vthe_t) - (F_{\sigma}^{-} + r) (\vthe_{t+1}) \right)}{\alpha} - \frac{2}{\alpha} \<\nabla F_{\sigma}^{-}(\vthe_{t+1}) - \nabla F_{\sigma}^{-}, \vthe_{t+1} - \vthe_t\> \\
    \le & 2\|\g_t - \nabla F_{\sigma}^{-} (\vthe_{t})\|^2 + 2 \|\nabla_{\sigma}^{-} (\vthe_{t}) - \nabla_{\sigma}^{-} (\vthe_{t+1})\|^2 + \frac{L}{\alpha}\| \vthe_{t+1} - \vthe_t \|^2 \\
    & + \frac{2 \left( (F_{\sigma}^{-} + r) (\vthe_t) - (F_{\sigma}^{-} + r) (\vthe_{t+1}) \right)}{\alpha} - \frac{2}{\alpha} \<\nabla F_{\sigma}^{-}(\vthe_{t+1}) - \nabla F_{\sigma}^{-}, \vthe_{t+1} - \vthe_t\> \\
    \le & 2\|\g_t - \nabla F_{\sigma}^{-} (\vthe_{t})\|^2 + \frac{2 \left( (F_{\sigma}^{-} + r) (\vthe_t) - (F_{\sigma}^{-} + r) (\vthe_{t+1}) \right)}{\alpha} + (2L^2 + \frac{3L}{\alpha}) \|\vthe_{t+1} - \vthe_{t}\|^2,
\end{align*}
where the second inequality is due to Young's inequality and the last inequality is due to the smoothness of the gradient. Summing up the above inequality over time steps $t=0,\dots, T-1$, we have :
\begin{align*}
    & \sum_{t=0}^{T-1} \| \g_t - \nabla F_{\sigma}^{-} (\vthe_{t+1}) + \frac{1}{\alpha} (\vthe_{t+1} - \vthe_t) \|^2 \\
    \le & 2 \sum_{t=0}^{T-1} \|\g_t - \nabla F_{\sigma}^{-} (\vthe_{t})\|^2 + \frac{2 \left( (F_{\sigma}^{-} + r) (\vthe_0) - (F_{\sigma}^{-} + r) (\vthe_{T-1}) \right)}{\alpha} + (2L^2 + \frac{3L}{\alpha}) \sum_{t=0}^{T-1} \|\vthe_{t+1} - \vthe_{t}\|^2 \\
    \le & 2 \sum_{t=0}^{T-1} \|\g_t - \nabla F_{\sigma}^{-} (\vthe_{t})\|^2 + \frac{4 B}{\alpha} + \frac{2}{\alpha^2} \sum_{t=0}^{T-1} \|\vthe_{t+1} - \vthe_{t}\|^2, 
\end{align*}
where the last inequality is due to Assumption \ref{ass:bound} and setting $\alpha = \frac{c}{L} < \frac{1}{2L}$. Combing the above inequality with Equation \eqref{eq2} and Equation \eqref{eq:frech}, we have :
\begin{align*}
    & \E[\text{dist} (0, \hat{\partial} (F_{\sigma}^{-} + r)(\vthe_{T}))^2] \\
    \le & \frac{1}{T} \sum_{t=0}^{T-1} \E \left[\| \g_t - \nabla F_{\sigma}^{-} (\vthe_{t+1}) + \frac{1}{\alpha} (\vthe_{t+1} - \vthe_t) \|^2 \right] \\
    \le & \frac{2}{T} \sum_{t=0}^{T-1} \E \|\g_t - \nabla F_{\sigma}^{-} (\vthe_{t})\|^2 + \frac{4 B}{T \alpha} + \frac{2}{\alpha^2 T} \sum_{t=0}^{T-1} \|\vthe_{t+1} - \vthe_{t}\|^2 \\
    \le & \frac{2}{T} \sum_{t=0}^{T-1} \E \|\g_t - \nabla F_{\sigma}^{-} (\vthe_{t})\|^2 + \frac{4 B}{T \alpha} + \frac{2}{\alpha^2 T} \left( \frac{4B}{\frac{1}{\alpha} - 2L } + \frac{1}{\frac{L}{\alpha} - 2L^2}\sum_{t=0}^{T-1} \E \|\g_t - \nabla F_{\sigma}^{-}(\vthe_t)\|^2 \right) \\
    = & \frac{2c(1-2c) + 2}{c(1-2c)} \frac{1}{T} \sum_{t=0}^{T-1} \E \|\g_t - \nabla F_{\sigma}^{-}(\vthe_t)\|^2 + \frac{12-8c}{1-2c} \frac{B}{\alpha T}
\end{align*}
     
We can now plug the result we obtained in Lemma~\ref{lem:var_total} in the result above, to obtain:

$$
\E\left[\operatorname{dist}\left(\bm{0}, \hat{\partial} \left( - F_{\sigma}\left(\vthe_T\right) + \mathds{1}_{L_0(k)}(\vthe_T)\right) \right)^2\right]  \leq \varepsilon^2,
$$

Using Jensen inequality and the fact that the square-root function is concave, we obtain Theorem \ref{thm:thm}.

\end{proof}